\documentclass[pmlr]{jmlr}
\usepackage{times}
\usepackage{bigints}
\usepackage{float}
\usepackage{algorithm2e,algorithmic}
\usepackage{pifont}

\newcommand{\N}{\mathbb{N}}
\newcommand{\R}{\mathbb{R}}
\newcommand{\ind}{\mathbb{I}}
\newcommand{\E}{\mathbb{E}}
\newcommand{\KL}{\mathrm{KL}}

\newcommand{\kl}{\mathrm{kl}}
\renewcommand{\P}{\mathbb{P}}
\renewcommand{\geq}{\geqslant}
\renewcommand{\leq}{\leqslant}

\renewcommand{\d}{\mathrm{d}}

\newcommand{\hmu}{\widehat{\mu}}

\newcommand{\astar}{a^{*}}
\newcommand{\mustar}{\mu^\star}
\newcommand{\Bpp}{\text{B}^{++}}
\newcommand{\Bp}{\text{B}^{+}}
\renewcommand{\phi}{\varphi}

\DeclareMathOperator*{\argmax}{arg\,max}
\DeclareMathOperator*{\loglog}{\log\!\log}
\DeclareMathOperator*{\logp}{\log_+}

\title[kl-UC$\Bpp$ algorithm]{A minimax and asymptotically optimal algorithm for stochastic bandits}

\author{\Name{Pierre M{\'e}nard} \Email{pierre.menard@math.univ-toulouse.fr}\AND
   \Name{Aur{\'e}lien Garivier} \Email{aurelien.garivier@math.univ-toulouse.fr}\\
   \addr Institut de Math\'ematiques de Toulouse; UMR5219\\
	Universit\'e de Toulouse; CNRS\\
	UPS IMT, F-31062 Toulouse Cedex 9, France }
\jmlrvolume{}
\jmlryear{2017}
\jmlrworkshop{Algorithmic Learning Theory 2017}
\editors{Steve Hanneke and Lev Reyzin}

\begin{document}
	
	\maketitle
	
	\begin{abstract}
		We propose the kl-UC$\Bpp$ algorithm for regret minimization in stochastic bandit models with exponential families of distributions.
		We prove that it is simultaneously asymptotically optimal (in the sense of Lai and Robbins' lower bound) and minimax optimal. 
		This is the first algorithm proved to enjoy these two properties at the same time.  
		This work thus merges two different lines of research with simple and clear proofs.
	\end{abstract}
	
	\section{Introduction}
	For regret minimization in stochastic bandit problems, two notions of time-optimality coexist.
	On the one hand, one may consider a fixed model: the famous lower bound by~\citet{lai1985asymptotically} showed that the regret of any consistent strategy should grow at least as $C(\mu)\log(T)\big(1-o(1)\big)$ when the horizon $T$ goes to infinity. Here, $C(\mu)$ is a constant depending solely on the model. 
	A strategy	with a regret upper-bounded by $C(\mu)\log(T)\big(1+o(1)\big)$ will be called in this paper \emph{asymptotically-optimal}.
	Lai and Robbins provided a first example of such a strategy in their seminal work. 
	Later, \citet{garivier2011kl} and \cite{MaillardMunosStoltz11klucb} provided finite-time analysis for variants of the UCB algorithm (see~\citet{Agrawal95,burnetas1996optimal,auer2002finite})  which imply asymptotic optimality. 
	Since then, other algorithms like Bayes-UCB~\citep{kaufmann2012bayesian} and Thompson Sampling~\citep{korda2013thompson} have also joined the family.
	
	On the other hand, for a fixed horizon $T$ one may assess the quality of a strategy by the greatest regret suffered in all possible bandit models. If the regret of a bandit strategy is upper-bounded by $C'\sqrt{KT}$ (the optimal rate: see~\citet{auer2002nonstochastic} and \citet{cesa2006prediction}) for some numeric constant $C'$, this strategy is called \emph{minimax-optimal}. The PolyINF and the MOSS strategies by~\cite{audibert2009minimax} were the first proved to be minimax-optimal.
	
	Hitherto, as far as we know, no algorithm was proved to be \emph{at the same time} asymptotically- and minimax-optimal. Two limited exceptions may be mentioned: the case of two Gaussian arms is treated in~\cite{garivier2016onexplore}; and the OC-UCB algorithm of~\citet{lattimore2015optimally} is proved to be minimax-optimal and almost problem-dependent optimal for Gaussian multi-armed bandit problems.
	Notably, the OC-UCB algorithm satisfies another worthwhile property of \emph{finite-time instance near-optimality}, see Section~2 of~\citet{lattimore2015optimally} for a detailed discussion.
	
	\paragraph{Contributions.}
	
	In this work, we put forward the kl-UC$\Bpp$ algorithm, a slightly modified version of kl-UC$\Bp$ algorithm discussed in~\cite{garivier2016onexplore} as an empirical improvement of UCB, and analyzed in~\citet{kaufmann2016bayesian}.  This bandit strategy is designed for some exponential distribution families, including for example Bernoulli and Gaussian laws. It borrows from the MOSS algorithm of~\cite{audibert2009minimax} the idea to divide the horizon by the number of arms in order to reach minimax optimality. We prove that it is at the same time asymptotically- and minimax-optimal. 
	This work thus merges the progress which has been made in different directions towards the understanding of the optimism principle, finally reconciling the two notions of time-optimality. 
	
	Insofar, our contribution answers a very simple and natural question. The need for simultaneous minimax- and problem-dependent optimality could only be addressed in very limited settings by means that could not be generalized to the framework adopted in our paper. Indeed, for a given horizon $T$, the worst problem depends on $T$: it involves arms separated by a gap of order $\sqrt{K/T}$. Treating the $T$-dependent problems correctly for all $T$ appears as a quite different task than catching the optimal, problem-dependent speed of convergence for every fixed bandit model. We show in this paper that the two goals can indeed be achieved simultaneously.
	
	
	Combining the two notions of optimality requires a modified exploration rate. We stick as much as possible to existing algorithms and methods, introducing just what is necessary to obtain the desired results. Starting from that of kl-UCB (so as to have a tight asymptotic analysis), one has to completely cancel the exploration bonus of the arms that have been drawn roughly $T/K$ times. The consequence is very slight and harmless in the case where the best arm is much better than the others, but essential in order to minimize the regret in the worst case where the best arm is barely distinguishable from the others. Indeed, when the best arm is separated by a gap of order $\sqrt{K/T}$ from the suboptimal arms, we can not afford to draw more than $T/K$ times a suboptimal arm so as to get a regret of order $\sqrt{KT}$.	
	
	We present a general yet simple proof, combining the best elements of the above-cited sources which are simplified as much as possible and presented in a unified way. To this end, we develop new deviation inequalities, improving the analysis of the different terms contributing to the regret. This analysis is made in the framework which we believe is the best compromise between simplicity and generality (simple exponential families). This permits us to treat, among others, the Bernoulli and the Gaussian case at the same time. More fundamentally, this appears to us as the right, simple framework for the analysis, which emphasizes what is really required to have simple lower- and upper-bounds (the possibility to make adequate changes of measure, and Chernoff-type deviation bounds).
	
%
%
	The paper is organized as follows. In Section~\ref{sec:settings}, we introduce the setting and assumptions required for the main results, Theorems~\ref{th:minimax_bound} and~\ref{th:asymptotic_bound}, which are presented in Section~\ref{sec:klucb_plus_plus_algo}. We give the entire proofs of these results in Sections~\ref{sec:proof:minimax} and~\ref{sec:proof:asymptotic}, with only a few technical lemmas proved in Appendix~\ref{app:technical}. We conclude in Section~\ref{sec:conclusion} with some brief references to possible future prospects.
	
	\section{Notation and Setting}\label{sec:settings}
	\paragraph{Exponential families.}
	We consider a simple stochastic bandit problem with $K$ arms indexed by $a \in \{1,\ldots,K\}$, with $K\geq 2$.
	Each arm is assumed to be a probability distribution of some canonical one-dimensional exponential family $\nu_{\theta}$ indexed by $\theta\in\Theta$. The probability law $\nu_\theta$ is assumed to be  absolutely continuous with respect to a dominating measure $\rho$ on $\R$, with a density given by
	\begin{equation*}
	\frac{\d \nu_{\theta}}{\d \rho}(x)=\exp(x \theta -b(\theta)), \text{\quad where } b(\theta)= \log \int_{\R} e^{x\theta}\d\rho(x)\,\text{ and }
	\Theta=\big\{\theta \in \R:\ b(\theta)<+\infty\big\}\,.
	\end{equation*}
	It is well-known that $b$ is convex, twice differentiable on $\Theta$, that $b'(\theta)=E(\nu_\theta)$ and $b''(\theta)=V(\nu_\theta)>0$ are respectively the mean and the variance of the distribution $\nu_\theta$.
	The family can thus be parametrized by the mean $\mu=b'(\theta)$, for $\mu\in I=b'(\Theta):=(\bar{\mu}^-,\bar{\mu}^+)$. The Kullback-Leibler divergence between two distributions is
	$   \KL(\nu_\theta,\nu_{\theta'})=b(\theta')-b(\theta)-b'(\theta)(\theta'-\theta)$. 
	This permits to define the following divergence on the set of arm expectations: for $\mu=E(\nu_\theta)$ and $\mu'=E(\nu_{\theta'})$, we write
	\begin{equation*}
	\kl(\mu,\mu'):=\KL(\nu_\theta,\nu_{\theta'})\;.
	\end{equation*}
	For a minimax analysis, we need to restrict the set of means to bounded interval: we suppose that each arm  $\nu_\theta$ satisfies $\mu=b'(\theta)\in [\mu-,\mu+]\subset I$ for two fixed real numbers $\mu^+,\mu^-$. 
	Our analysis requires a Pinsker-like inequality;  we therefore assume that the variance is bounded in the exponential family: there exists $V>0$ such that
	\[
	\sup_{\mu\in I} b''\big( {b'}^{-1}(\mu)\big)=\sup_{\mu \in I} V\big(\nu_{{b'}^{-1}(\mu)}\big)\leq V<+\infty\;.
	\]
	This implies that for all $\mu,\mu'\in I$,
	\begin{equation}
	\label{eq:pinsker}
	\kl(\mu,\mu')\geq \frac{1}{2 V}(\mu-\mu')^2\,.
	\end{equation}
	In the sequel, we denote by $\mathcal{F}$ the set of bandit problems $\nu$ satisfying these assumptions.
	By the usual Pinsker inequality, this setting includes in particular Bernoulli bandits with $V=1/4$ and $\kl(\mu,\mu')=\mu\log(\mu/\mu')+(1-\mu)\log\big((1-\mu)/(1-\mu')\big)$ (by convention, $0\log 0=0\log 0/0=0$).
	This also includes (bounded) Gaussian bandits with known variance $\sigma^2$, with the choice $V=\sigma^2$ and $\kl(\mu,\mu')=(\mu-\mu')^2/(2\sigma^2)$.
	
	\paragraph{Regret.}
	The $K$ arms are denoted $\nu_{\theta_1},\dots\nu_{\theta_K}$, and the expectation of arm $a\in\{1,\dots,K\}$ is denoted by $\mu_a$.
	At each round $1\leq t\leq T$, the player pulls an arm $A_t$ and receives an independent draw $Y_t$ of the distribution $\nu_{\theta_{A_t}}$. This reward is the only piece of information
	available to the player.
	The best mean is $\mu^\star = \max_{a=1,\ldots,K} \mu_a$.
	We denote by $N_a(T) = \sum_{t=1}^T \ind_{\{ A_t = a \}}$ the number of draws of arm $a$ up to and including time $T$.
	In this work, the goal is to minimize the \emph{expected regret}
	\[
	R_T = T\mu^\star - \E\!\left[ \sum_{t=1}^T Y_t \right]
	= \E\!\left[ \sum_{t=1}^T \bigl( \mu^\star - \mu_{A_t} \bigr) \right]
	= \sum_{a=1}^K \big(\mustar -\mu_a\big) \, \E  \bigl[ N_a(T) \bigr]\,.
	\]
	\cite{lai1985asymptotically} proved that if a strategy is uniformly efficient, that is if it is such that under any bandit model of a sufficiently rich family (such as an exponential family described above) $R_T=o(T^\alpha)$ holds for every $\alpha>0$, then it needs to draw any suboptimal arm $a$ at least as often as
	\[\E\big[N_a(T)\big] \geq \frac{\log(T)}{\kl(\mu_a,\mustar)}\;\big(1-o(1)\big)\;.\]
	In light of the previous equality, this directly implies an asymptotic lower bound on $R_T/\log(T)$.
	
	On the other side, a straightforward adaptation of the the proof of Theorem A.2 of \citet{auer2002nonstochastic} shows that there exists a constant $C'$ depending only on 
	the considered family $\mathcal{F}$ of distributions such that 
	\begin{equation*}
	\sup_{\nu\in \mathcal{F}} R_T\geq C'\min\big(\sqrt{K T },T\big)\,,
	\end{equation*}
	where the supremum is taken over all bandit problems $\nu$ in $\mathcal{F}$. Note that the notion of minimax-optimality is defined here up to a multiplicative constant, in contrast to the definition of (problem-dependent) asymptotic optimality. 
	For a discussion on the minimax and asymptotic lower bounds, we refer to \citet{garivier2016explore} and references therein.
	
	\section{The kl-UC$\Bpp$ Algorithm}
	\label{sec:klucb_plus_plus_algo}
	We denote by $\hmu_{a,n}$ the empirical mean of the first $n$ rewards from arm $a$. 
	The empirical mean of arm $a$ after $t$ rounds is 
	\begin{equation*}
	\hmu_a(t)=\hmu_{a,N_a(t)}=\frac{1}{N_a(t)}\sum_{s=1}^{t} Y_s\, \ind_{\{A_s=a\}}\,.
	\end{equation*}
	\begin{center} \fbox{\begin{minipage}{12cm}
				\begin{algorithm}[H]
					\label{alg:KLUCB_plus_plus}
					\SetKwInput{para}{Parameters}
					\para{The horizon $T$ and an  exploration function $g: \N\mapsto\R^+$.}
					\SetKwInput{init}{Initialization}
					\init{Pull each arm of $\{1,..,K\}$ once.}
					\smallskip
					\textbf{For} $t= K$ to $T-1$, \textbf{do}
					\smallskip
					\begin{enumerate}
						\item Compute for each arm $a$ the quantity 
						\begin{equation}
						U_a(t)=\sup \Bigg\{ \mu\in I\ :\ \kl\big(\hmu_a(t),\mu\big)\leq \frac{g\big(N_a(t)\big)}{N_a(t)}\Bigg\}\,.
						\end{equation}
						\item Play $A_{t+1}\in \argmax_{a\in \{1,..,K\}} \;U_a(t)$.
					\end{enumerate}
				\end{algorithm}
		\end{minipage}}
	\end{center}
	
	The kl-UC$\Bpp$ algorithm  is a slight modification of algorithm kl-UC$\Bp$ of \citet{garivier2011kl} and of the kl-UCB-$\text{H}^+$ analyzed in \citet{kaufmann2016bayesian}. It uses the exploration function $g$ given by  
	\begin{equation}
	\label{eq:def_g}
	g(n)=\logp\!\!\Bigg( \frac{T}{K n}\Bigg(\log_{+}^2\!\!\Bigg(\frac{T}{K n}\Bigg)+1\Bigg)\Bigg)\,,
	\end{equation}
	where $\logp(x):= \max\big(\log(x), 0\big)$. The exploration function $g$ borrows the general form with the extra exploration rate from the kl-UCB algorithm, the division by the number of draws from kl-UC$\Bp$, and the division by the number of arm from MOSS.   	
	
	\noindent The following results state that the kl-UC$\Bpp$ algorithm is simultaneously minimax- and {a\-symp\-to\-ti\-cally}-optimal.
	\begin{theorem}[Minimax optimality]
		\label{th:minimax_bound}
		For any family $\mathcal{F}$ satisfying the assumptions detailed in Section~\ref{sec:settings}, and for any bandit model $\nu\in\mathcal{F}$, the expected regret of the kl-UC$\Bpp$ algorithm is upper-bounded as
		\begin{equation}
		\label{eq:minimax_bound}
		R_T\leq 76 \sqrt{ V K T}+(\mu^+-\mu^-)K\,.
		\end{equation}
	\end{theorem}
	
	\begin{theorem}[Asymptotic optimality]
		\label{th:asymptotic_bound}
		For any bandit model $\nu\in\mathcal{F}$, for any suboptimal arm $a$ and any $\delta$ such that $\sqrt{22 V K / T}\leq \delta\leq(\mustar-\mu_a)/3$,
		\begin{equation}
		\label{eq:asymptotic_bound}
		\E\big[N_a(T)\big]\leq \frac{\log(T)}{\kl(\mu_a+\delta,\mu^\star-\delta)}+O\left(\frac{\loglog(T)}{\delta^2}\right)\,
		\end{equation}
		which implies the asymptotic optimality (see the end of the proof in Section~\ref{sec:proof:asymptotic} for an explicit bound).
	\end{theorem}
	Theorems~\ref{th:minimax_bound} and~\ref{th:asymptotic_bound} are proved in Sections~\ref{sec:proof:minimax} and~\ref{sec:proof:asymptotic} respectively. 
	The main differences between the two proofs are discussed at the beginning of Section~\ref{sec:proof:asymptotic}.			
	Note that the two regret bounds of Theorems~\ref{th:minimax_bound} and~\ref{th:asymptotic_bound} also apply to all $[0,1]$-valued bandit models, with the value $V=1/4$, as the deviations of $[0,1]$-valued random variables are dominated by those of a Bernoulli distribution with the same mean (this is discussed for example in~\cite{cappe2013kullback}). However, the kl-UC$\Bpp$ algorithm is not asymptotically optimal then: the regret bound in $\log(T)/\kl(\mu_a,\mu^*)$ is not optimal in that case. Asymptotic optimality would require tight distribution-dependent, non-parametric upper confidence bounds (for example based on the empirical-likelihood method, as in the above cited paper). This is out of the scope of this work (and would require a lot more space).
	
	\section{Proof of Theorem~\ref{th:minimax_bound}}
	\label{sec:proof:minimax}
	This proof merges merges ideas presented in~\citet{bubeck2013prior} for the analysis of the MOSS algorithm and from the analysis of kl-UCB in~\citet{cappe2013kullback} (see also~\citet{kaufmann2016bayesian}). It is divided into the following steps:
	
	\paragraph{Decomposition of the regret.}
	Let $\astar$ be the index of an optimal arm. Since by definition of the strategy $U_{\astar}(t)\leq U_{A_{t+1}}(t)$ for all $t\geq K-1$, the regret can be decomposed as follows:
	\begin{equation}
	\label{eq:decomp_regret}
	R_T\leq K (\mu^+-\mu^-) + \underbrace{\sum_{t=K}^{T-1} \E\big[\mustar-U_{\astar}(t)\big]}_{A}+\underbrace{\sum_{t=K}^{T-1} \E\big[U_{A_{t+1}}(t)-\mu_{A_{t+1}}\big]}_{B}\,.
	\end{equation}
	We define $\delta_0=\sqrt{22 V K /T}$; since the bound \eqref{eq:minimax_bound} is otherwise trivial, we assume in the sequel that $\delta_0\leq 1$. For the first term $A$, as in the proof of MOSS algorithm, we carefully upper bound the probability that appears inside the integral thanks to a 'peeling trick'. The second term B is easier to handle since we can reduce the index to UCB-like-index thanks to the Pinsker inequality~\eqref{eq:pinsker} and proceed as in~\citet{bubeck2013prior}.

	\paragraph{Step 1: Upper-bounding $A$.}
	Term $A$ is concerned with the optimal arm $\astar$ only. Two words of intuition: since $U_{\astar}(t)$ is meant to be an upper confidence bound for $\mustar$, this term should not be too large, at least as long as the the confidence level controlled by function $g$ is large enough -- but when the confidence level is low, the number of draws is large and deviations are unlikely.
	
	Upper-bounding term $A$ boils down to controlling the probability that $\mustar$ is under-estimated at time $t$. Indeed,
	\begin{align}
	\E\big[\mustar-U_{\astar}(t)\big]\leq \E\Big[\big(\mustar-U_{\astar}(t)\big)_+\Big] & \leq \int_{0}^{+\infty} \P\big(u < \mustar-U_{\astar}(t)\big)\d u \nonumber\\
	&\leq \delta_0 +\int_{\delta_0}^{+\infty} \P\big(U_{\astar}(t)\leq \mustar-u\big)\d u\,,
	\label{eq:bound_A_int}
	\end{align}
	and we need to upper bound the left-deviations of the mean of arm $\astar$.
	On the event $\{U_{\astar}(t)\leq \mustar-u\}$, we have that $\hmu_{\astar}(t)\leq U_{\astar}(t)\leq \mustar-u<\mustar$, and by definition of $U_{\astar}(t)$ it holds that
	\begin{equation*}
	\kl\big(\hmu_{\astar}(t),\mustar\big)\geq \frac{g\big(N_{\astar}(t)\big)}{N_{\astar}(t)}\,.
	\end{equation*}
	Consequently,
	\begin{align}
	\P\big(U_{\astar}(t)\leq \mustar-u\big)&\leq \P\Big(\hmu_{\astar}(t)\leq \mustar-u\ \text{ and }\ \kl\big(\hmu_{\astar}(t),\mustar\big)\geq g\big(N_{\astar}(t)\big)/N_{\astar}(t) \Big) \nonumber\\
	&\leq \P\big(\exists 1\leq n\leq T,\ \  \hmu_{\astar,n}\leq \mustar-u\ \text{ and }\ \kl(\hmu_{\astar,n},\mustar)\geq g(n)/n \big)\,. \label{eq:proba_in_A}
	\end{align}
	For small values of $n$, the dominant term is given by $\kl(\hmu_{\astar,n},\mustar)\geq g(n)/n$, whereas for large $n$ the event $\hmu_{\astar,n}\leq \mustar-u$ is quite unlikely. This is why we split the probability in two terms, proceeding as follows. Let $f$ be the function defined, for $u\geq \delta_0$, by
	\begin{equation*}
	f(u)=\frac{2V}{u^2}\log\!\!\left(\frac{T u^2}{2 V K}\right)\,.
	\end{equation*}
	Our choice of $\delta_0$ implies that $f(u)K/T\leq \exp(-3/2)$, and thus
	\begin{equation}
	\label{eq:bound_on_f}
	f(u)< \frac{T}{K} \qquad\text{ and }\qquad\log\!\!\left(\frac{T}{K f(u)}\right)\geq 3/2\,.
	\end{equation}
	In particular, for $n\leq f(u)$ it holds that
	\begin{equation*}
	g(n)= \log\!\!\left(\frac{T}{K n} \left( 1+\log^2\!\!\left(\frac{T}{K n}\right)\right)\right)\,.
	\end{equation*}
	It appears that $f(u)$ is the right place where to split the probability of Equation~\eqref{eq:proba_in_A}: defining $\kl_+(p,q):=\kl(p,q)\ind_{\{p\leq q\}}$, we write
	\begin{align}
	\P\big(\exists 1 &\leq n\leq T,\ \  \hmu_{\astar,n}\leq \mustar-u\ \text{ and }\ \kl(\hmu_{\astar,n},\mustar)\geq g(n)/n \big)\leq \nonumber\\
	&\underbrace{\P\big( \exists 1\leq n\leq f(u),\ \ \kl_+(\hmu_{\astar,n}, \mustar)\geq g(n)/n \big)}_{A_1}+\underbrace{\P\big( \exists f(u)\leq n\leq T,\ \ \hmu_{\astar,n}\leq \mustar-u\big)}_{A_2}\,.\label{eq:def_A1_A2}
	\end{align}
	Controlling terms $A_1$ and $A_2$ is a matter of deviation inequalities. \\
	\textbf{Step 1.1: Upper-bounding $A_1$. } The term $A_1$, which involves  self-normalized devation probabilities, can be upper-bounded thanks to a 'peeling trick' as in the proof of Theorem~5 from \citet{audibert2009minimax}. 
	We assume that $f(u)\geq 1$, for otherwise $A_1=0$. We use the grid $f(u)/\beta^{\ell+1}\leq n\leq f(u)/\beta^{\ell}$, where the real $\beta>1$ will be chosen later. We write 
	\begin{equation}
	A_1\leq \sum_{\ell=0}^{+\infty} \underbrace{\P\Bigg( \exists \frac{f(u)}{\beta^{\ell+1}}\leq n\leq  \frac{f(u)}{\beta^{\ell}},\ \ \kl_+(\hmu_{\astar,n},\mustar)\geq \gamma_\ell \Bigg)}_{A_1^\ell} \,,
	\label{eq:peeling_A_1}
	\end{equation}
	where
	\begin{equation*}
	\gamma_\ell=\frac{\log\!\!\Bigg(\frac{T \beta^\ell}{K f(u)}\Bigg(1+\log^2\!\!\Bigg(\frac{T}{K f(u)}\Bigg)\Bigg)}{f(u)/\beta^\ell}\,.
	\end{equation*}
	Thanks to Doob's maximal inequality (see Lemma~\ref{lem:maximal_inequality} in Appendix~\ref{app:technical}), 
	\begin{align*}
	A_1^\ell  \leq \exp\!\! \Bigg( -\frac{f(u)}{\beta^{\ell+1}}\,\gamma_\ell \Bigg)= e^{-\ell \log(\beta)/\beta -C/\beta}\,,
	\end{align*}
	where 
	\begin{equation}
	C:=\log\!\!\Bigg(\frac{T}{ K f(u)} \Bigg(1+\log^2\!\!\Bigg(\frac{T}{K f(u)}\Bigg)\Bigg)\Bigg)\,.
	\end{equation}
	Plugging this last inequality into \eqref{eq:peeling_A_1}, together with the numerical inequality of Lemma~\ref{lem:bound_max} (see Appendix~\ref{app:technical}), we get 
	\begin{multline*}
	A_1 \leq \sum_{\ell=0}^{+\infty} e^{-\ell \log(\beta)/\beta-C/\beta}
	= \frac{1}{1-e^{-\log(\beta)/\beta}}e^{-C/\beta}\\
	\leq \frac{e}{e^{\log(\beta)/\beta}-1} e^{-C/\beta}
	\leq 2e \max\big(\beta,\beta/(\beta-1)\big) e^{-C/\beta}\,.
	\end{multline*}
	But thanks to Equation~\eqref{eq:bound_on_f},
	\begin{equation*}
	C=\log\!\!\Bigg(\frac{T}{ K f(u)} \Bigg(1+\log^2\!\!\Bigg(\frac{T}{K f(u)}\Bigg)\Bigg)\Bigg)\geq \log\!\!\Bigg(\frac{T}{ K f(u)}\Bigg) \geq \frac{3}{2}\,.
	\end{equation*}
	It is now time to choose $\beta:=C/(C-1)$, so that $\beta\leq 2 C$ and $\beta/(\beta-1)=C$. Together with the definition of $f$, this choice yields
	\begin{align}
	A_1 &\leq 4 e^2 C e^{-C} 
	=4e^2 \frac{\log\!\!\Bigg(\frac{T}{ K f(u)} \Bigg(1+\log^2\!\!\Bigg(\frac{T}{K f(u)}\Bigg)\Bigg)\Bigg)}{1+\log^2\!\!\Bigg(\frac{T}{K f(u)}\Bigg)} \frac{K f(u)}{T}\,,\label{eq:inequality_lemma}
	\end{align}
	and therefore 
	\begin{equation}\label{eq:bound_A_1}
	A_1\leq 4e^2 \frac{K f(u)}{T} = \frac{16 e^2 V K}{T u^2}\log\!\!\left(\sqrt{\frac{T}{2 V K}}u\right)\,
	\end{equation}
	as, for all $x\geq 1$,
	\begin{equation*}
	\frac{\log\!\Big(x\big(1+\log^2(x)\big)\Big)}{1+\log^2(x)}\leq 1\,.
	\end{equation*}
	\textbf{Step 1.2: Upper-bounding $A_2$}. The term $A_2$ is more simple to handle, as it does not involve self-normalized deviations. Thanks to the maximal inequality (recalled in Equation~\eqref{eq:hoeffding} of Appendix~\ref{app:technical}) and thanks to the Pinsker-like inequality~\eqref{eq:pinsker}, 
	\begin{align}
	A_2 \leq e^{-u^2f(u)/2V}=\frac{2 V K}{T u^2} \label{eq:bound_A_2}\,.
	\end{align}
	Putting Equations~\eqref{eq:bound_A_int} to \eqref{eq:bound_A_2} together, 
	we obtain that
	\begin{equation}
	\E\big[\mustar-U_{\astar}(t)\big]\leq \delta_0 +\int_{\delta_0}^{+\infty} \frac{16 e^2 V K}{T u^2}\log\!\!\left(\sqrt{\frac{T}{2 V K}}u\right)+\frac{2 V K}{T u^2} \d u\,.
	\label{eq:A_with_A_1and_A_2}
	\end{equation}
	It remains only to conclude with some calculus:
	\begin{align*}
	\int_{\delta_0}^{+\infty} \frac{16 e^2 V K}{T u^2}\log\!\!\left(\sqrt{\frac{T}{2 V K}}u\right)\d u &=\left[-\frac{16 e^2 V K}{T u}\log\!\!\left(e\sqrt{\frac{T}{2 V K}}u\right)\right]_{\delta_0}^{+\infty}\\
	&= \frac{16 e^2 \sqrt{V}}{\sqrt{22}} \log\big(e\sqrt{11}\big)\sqrt{\frac{K}{T}}\,.
	\end{align*}
	Similarly,
	\begin{equation*}
	\int_{\delta_0}^{+\infty} \frac{2 V K}{T u^2} \d u= 2\sqrt{\frac{V}{22}}\sqrt{\frac{K}{T}}\,,
	\end{equation*}
	and replacing $\delta_0$ by its value we obtain from Equation~\eqref{eq:A_with_A_1and_A_2} the following relation:
	\begin{equation*}
	\E\big[\mustar-U_{\astar}(t)\big]\leq \sqrt{V}\left(\sqrt{22}+\frac{16 e^2}{\sqrt{22}} \log\big(e\sqrt{11}\big)+\frac{2}{\sqrt{22}} \right)\sqrt{\frac{K}{T}}\,.
	\end{equation*}
	Summing over $t$ from $K$ to $T-1$, this yields:
	\begin{equation}
	A\leq \sqrt{V}\left(\sqrt{22}+\frac{16 e^2}{\sqrt{22}} \log\big(e\sqrt{11}\big)+\frac{2}{\sqrt{22}} \right)\sqrt{K T}\,.
	\label{eq:bound_A}
	\end{equation}
	
	\paragraph{Step 2: Upper-bounding $B$.} 
	Term $B$ is of different nature, since typically $U_{A_{t+1}}(t)>\mu_{A_{t+1}}$.
	However, as for the term $A$, we first reduce the problem to the upper-bounding of a probability:
	\begin{align}
	B&\leq \sum_{t=K}^{T-1} \delta_0+\int_{\delta_0}^{+\infty}\P\big( U_{A_{t+1}}(t)-\mu_{A_{t+1}}\geq  u\big) \d u \nonumber\\
	&\leq  T\delta_0+ \int_{\delta_0}^{+\infty} \sum_{t=K}^{T-1} \P\big( U_{A_{t+1}}(t)-\mu_{A_{t+1}}\geq u \big)
	\d u\,.\label{eq:bound_B_int}
	\end{align}
	The event $\big\{U_{A_{t+1}}(t)-\mu_{A_{t+1}}\geq u\big\}$ is typical if $N_{A_{t+1}}(t)$ is small, and corresponds to a deviation of the sample mean otherwise.
	In order to handle this correctly, we first get rid of the randomness of $N_{A_{t+1}}(t)$ by the pessimistic trajectorial upper bound from \citet{bubeck2013prior}
	\begin{equation*}
	\sum_{t=K}^{T-1} \ind_{\big\{ U_{A_{t+1}}(t)-\mu_{A_{t+1}}\geq u\big\}}\leq \sum_{n=1}^{T}\sum_{a=1}^{K} \ind_{\big\{ U_{a,n}-\mu_{a}\geq u\big\}}\,.
	\end{equation*}
	In addition, we simplify the upper bound thanks to our assumption~\eqref{eq:pinsker} that some Pinsker type inequality is available:
	\begin{equation}
	U_{a,n}:=\sup \Bigg\{ \mu\in I\ :\ \kl\big(\hmu_{a,n},\mu\big)\leq \frac{g(n)}{n}\Bigg\}\leq B_{a,n}:= \hmu_{a,n} +\sqrt{2 V\frac{g(n)}{n}}\,.
	\end{equation}
	Hence, $B$ can be upper-bounded as
	\begin{equation}
	B\leq T\delta_0+ \sum_{a=1}^K\int_{\delta_0}^{+\infty} \sum_{n=1}^{T} \P( B_{a,n}-\mu_{a}\geq u)\d u\,.
	\label{eq:Bound_B_wtih_UCB_index}    
	\end{equation}
	Then, we need only to upper bound $\sum_{n=1}^{T} \P( B_{a,n}-\mu_{a}\geq u)$ for each arm $a\in\{1,\dots,K\}$. We cut the sum at the critical sample size $n(u)$ where the event $\big\{B_{a,n}-\mu_a>u\big\}$ becomes atypical: for $u\geq \delta_0$, let $n(u)$ be the integer such that 
	\begin{equation*}
	n(u)= \left\lceil\frac{8V}{u^2}\log\!\!\left(\frac{T u^2}{8 V K}\right)\right\rceil\,.
	\end{equation*}
	For $n\geq n(u)$ it holds that
	\begin{equation}
	\label{eq:bound_exp_term_nu}
	\sqrt{2V\frac{g(n)}{n}}\leq \frac{u}{\sqrt{2}}\,.
	\end{equation}
	Indeed, as $\log(1+x^2)\leq x$ for all $x\geq 0$, we have
	\begin{equation*}
	2V\frac{g(n)}{n}\leq \frac{4V}{n}\logp\!\!\left(\frac{T}{Kn}\right)\,,
	\end{equation*}
	also observe that $h(x):=\log\big(x/ \log(x)\big)/\log(x)$ is such that $h(x)\leq 1$ for $x\geq 11/4$, and thus for $n\geq n(u)$ and $u\geq \delta_0$
	\begin{align*}
	2V\frac{g(n)}{n}\leq \frac{4V}{n(u)}\logp\!\!\left(\frac{T}{K n(u)}\right)&\leq \frac{u^2}{2} h\left(\frac{T u^2}{8 V K}\right)
	\leq \frac{u^2}{2}\,.
	\end{align*}
	Therefore, cutting the sum in~\eqref{eq:Bound_B_wtih_UCB_index} at $n(u)$, we obtain:
	\begin{align}
	\sum_{n=1}^{T} \P( B_{a,n}-\mu_{a}\geq u)&\leq n(u)-1+\sum_{n=n(u)}^{T} \P\big( \hmu_{a,n}-\mu_{a}\geq u-\sqrt{2 V g(n)/n}\big)\nonumber\\
	&\leq n(u)-1 + \sum_{n=n(u)}^{T} \P\Big( \hmu_{a,n}-\mu_{a}\geq u\big(1-1/\sqrt{2}\big)\Big)\nonumber\\
	&\leq \frac{8V}{u^2}\log\!\!\left(\frac{T u^2}{8 V K}\right) + \sum_{n=n(u)}^{T} \P( \hmu_{a,n}-\mu_{a}\geq c u)\,, \label{eq:bound_after_cut}
	\end{align}
	where $c:=1-1/\sqrt{2}$.
	It remains to integrate Inequality~\eqref{eq:bound_after_cut} from  $u=\delta_0$ to infinity. 
	The first summand involves the same integral as we have already met in the upper bound of term $A_1$: 
	\begin{align*}
	\int_{\delta_0}^{+\infty} \frac{8V}{u^2}\log\!\!\left(\frac{T u^2}{8 V K}\right)\d u = 16 \sqrt{\frac{V}{22}} \log\!\!\left(e\sqrt{\frac{11}{4}}\right)\sqrt{\frac{T}{K}}\,.
	\end{align*}
	For the remaining summand, Inequality~\eqref{eq:hoeffding} yields  
	\begin{align*}
	\sum_{n=n(u)}^{T} \P( \hmu_{a,n}-\mu_{a}\geq c u)&\leq \sum_{n=n(u)}^{T} e^{- \frac{u^2 c^2 n}{2V}} \leq \frac{1}{e^{\frac{u^2 c^2}{2V}}-1}\,.
	\end{align*}
	Thus, as $e^{x}-1\geq x$ for all $x\geq 0$,
	\begin{align*}
	\int_{\delta_0}^{+\infty} \frac{1}{e^{\frac{u^2 c^2}{2V}}-1}\d u &\leq  \int_{\delta_0}^{+\infty} \frac{2 V}{u^2 c^2} \d u = \frac{2}{c^2}\sqrt{\frac{V}{22}} \sqrt{\frac{T}{K}}\,,
	\end{align*}
	Putting everything together starting from Inequality~\eqref{eq:bound_after_cut}, we have proved that
	\begin{equation*}
	\int_{\delta_0}^{+\infty} \sum_{n=1}^{T} \P( B_{a,n}-\mu_{a}\geq u) \d u \leq\sqrt{ \frac{V}{22}}\left(16\log\!\!\left(e\sqrt{\frac{11}{4}}\right)+\frac{2}{c^2}\right) \sqrt{\frac{T}{K}} \,.
	\end{equation*}
	By Equation~\eqref{eq:Bound_B_wtih_UCB_index}, replacing $\delta_0$ by its value finally yields
	\begin{equation}
	B\leq \sqrt{V}\Bigg(\sqrt{22}+\frac{16}{\sqrt{22}}\log\!\!\left(e\sqrt{\frac{11}{4}}\right)+\frac{2}{\sqrt{22}c^2}\Bigg) \sqrt{KT}\,.
	\label{eq:bound_B}
	\end{equation}
	\paragraph{Conclusion of the proof.}
	It just remains to plug Inequalities~\eqref{eq:bound_A} and~\eqref{eq:bound_B} into Equation \eqref{eq:decomp_regret}:
	\begin{align*}
	A+B &\leq \sqrt{V}\Bigg(2\sqrt{22}+\frac{16 e^2}{\sqrt{22}} \log\big(e\sqrt{11}\big)+\frac{2}{\sqrt{22}}+\frac{16}{\sqrt{22}}\log\!\!\left(e\sqrt{\frac{11}{4}}\right)+\frac{2}{\sqrt{22}c^2} \Bigg)\sqrt{K T}\\
	&\leq 76\sqrt{V K T}\,,
	\end{align*}
	which concludes the proof.
	
	\section{Proof of Theorem~\ref{th:asymptotic_bound}}
	\label{sec:proof:asymptotic}
	The analysis of asymptotic optimality shares many elements with the minimax analysis, with some differences however. The decomposition of the regret into two terms $A$ and $B$ is similar, but localized on a fixed sub-optimal arm $a\in\{1,\dots,K\}$: we analyze the number of draws of $a$ and not directly the regret (and we do not need to integrate the deviations at the end).		
	We proceed roughly as in the proof of Theorem~\ref{th:minimax_bound} for term $A$, which involves the deviations of an optimal arm.
	For term B, which stands for the behavior of the sub-optimal arm $a$, a different (but classical) argument is used, as one cannot simply use the Pinsker-like Inequality~\eqref{eq:pinsker} if one wants to obtain the correct constant (and thus asymptotic optimality).
	
	\paragraph{Decomposition of $\E \big[ N_a(T)\big ]$.} If arm $a$ is pulled at time $t+1$, then by definition of the strategy $U_{a^*}(t)\leq U_{a}(t)$ for any index $a^*$ of an optimal arm. Thus, for any fixed $\delta$ to be chosen later,
	\begin{align*}
	\big\{A_{t+1}=a\big\}&\subseteq\big\{ \mu^*-\delta \geq   U_{a}(t)\big\} \cup \{\mu^*-\delta < U_{a}(t) \text{ and } A_{t+1}=a\big\}\\
	&\subseteq \big\{ \mu^*-\delta \geq   U_{a^*}(t)\big\} \cup \{\mu^*-\delta < U_{a}(t) \text{ and } A_{t+1}=a\big\}\,.
	\end{align*}
	As a consequence,
	\begin{equation}
	\label{eq:decomposition_ENa}
	\E \big[ N_a(T)\big ] \leq 1+\underbrace{\sum_{t=K}^{T-1} \P\big( U_{a^*}(t)\leq \mu^*-\delta\big)}_{\text{A}}+\underbrace{\sum_{t=K}^{T-1} \P\big( \mu^*-\delta < U_{a}(t) \text{ and } A_{t+1}=a\big)}_{\text{B}}\,,
	\end{equation}
	and it remains to bound each of these terms.\\
	\paragraph{Step 1: Upper-bounding A.} As in the proof of Theorem~\ref{th:minimax_bound}, we write
	\begin{align}
	\P\big( U_{a^*}(t)&\leq \mu^*-\delta\big) \leq \nonumber\\
	&\underbrace{\P\big( \exists 1\leq n\leq f(\delta),\ \ \kl_+(\hmu_{\astar,n}, \mustar)\geq g(n)/n \big)}_{A_1}+\underbrace{\P\big( \exists f(\delta)\leq n\leq T,\ \ \hmu_{\astar,n}\leq \mustar-\delta \big)}_{A_2}\,,\label{eq:def_A1_A2_asymp}
	\end{align}
	where we use the same function
	\begin{equation*}
	f(\delta)=\frac{2V}{\delta^2}\log\!\!\left(\frac{T \delta^2}{2 K V}\right)\,.
	\end{equation*}
	Thanks to the Inequality~\eqref{eq:inequality_lemma} that we saw in the proof of Theorem~\ref{th:minimax_bound}, we obtain that
	\begin{align}
	A_1 \leq 4e^2 \frac{\log\!\!\Bigg(\frac{T}{ K f(\delta)} \Bigg(1+\log^2\!\!\Bigg(\frac{T}{K f(\delta)}\Bigg)\Bigg)\Bigg)}{\log\!\!\Bigg(\frac{T}{K f(\delta)}\Bigg)} \frac{f(\delta)}{\log\!\!\Bigg(\frac{T}{K f(\delta)}\Bigg)}\frac{K}{T}\nonumber
	\leq \frac{16e^2}{\delta^2}\frac{2 V K}{T}\,.\label{eq:bound_A_1_asymp}
	\end{align}
	Here, we used that for all $x\geq e^{3/2}$, since the condition $\delta^2\geq 22 V K / T$ implies that $f(\delta)K/T\leq e^{-3/2}$,
	
	\begin{equation*}
	\frac{\log\Big(x\big(1+\log^2(x)\big)\Big)}{\log(x)}\leq 2 \qquad\text{ and }\qquad\frac{\log(x)}{\log\big(x/\log(x)\big)}\leq 2\,,
	\end{equation*}
	and that
	\begin{equation*}
	\frac{f(\delta)}{\log\!\!\bigg(\frac{T}{K f(\delta)}\bigg)}=\frac{2 V }{\delta^2}\frac{\log\!\!\bigg(\frac{T\delta^2}{2 V K}\bigg)}{\log\!\!\bigg(\frac{T\delta^2}{2 V K}\frac{1}{\log\big(T\delta^2/(2 V K)\big)}\bigg)}\,.
	\end{equation*}
	Thanks to the maximal inequality recalled in Appendix~\ref{app:technical} as Equation~\eqref{eq:hoeffding}, it holds that
	\begin{align}
	A_2 \leq e^{-\delta^2f(\delta)/(2V)}=\frac{2 V K}{T \delta^2} \label{eq:bound_A_2_asymp}\,.
	\end{align}
	Putting Equations~\eqref{eq:def_A1_A2_asymp} to \eqref{eq:bound_A_2_asymp} together yields:
	\begin{equation}
	\label{eq:bound_A_asymp}
	A\leq (16e^2+1)\frac{2 V K}{\delta^2}\,.
	\end{equation}\\
	\paragraph{Step 2: Upper-bounding B.} Thanks to the definition of $U_a(t)$ it holds  that
	\begin{equation*}
	\big\{\mu^*-\delta < U_{a}(t) \text{ and } A_{t+1}=a\big\} \subseteq \Big\{\kl\big(\hmu_a(t),\mu^*-\delta\big)\leq g\big(N_a(t)\big)/N_a(t) \text{ and } A_{t+1}=a\Big\}   
	\end{equation*}
	Together with the following classical argument for regret analysis in bandit models, this yields:
	\begin{align}
	B&\leq \sum_{t=K}^{T-1} \P\big(\kl\big(\hmu_a(t),\mu^*-\delta\big)\leq g\big(N_a(t)\big)/N_a(t)  \text{ and } A_{t+1}=a\big)\nonumber\\
	&\leq \sum_{n=1}^{T} \P\big(\kl(\hmu_{a,n},\mu^*-\delta)\leq g(n)/n \big)\nonumber\\
	&\leq  \sum_{n=1}^{T} \P\bigg(\kl(\hmu_{a,n},\mu^*-\delta)\leq \log\!\!\Big(T/K\big(1+\log^2(T/K)\big)\Big)/n \bigg)\,, \label{eq:bound_B_n}
	\end{align}
	as it holds $g(n) \leq g(1)$. Now, let $n(\delta)$ be the integer defined as
	\begin{equation*}
	n(\delta)= \left\lceil \frac{\log\Big(T/K\big(1+\log^2(T/K)\big)\Big)}{\kl(\mu_a+\delta,\mu^*-\delta)} \right\rceil\,.
	\end{equation*}
	Then, for $n\geq n(\delta)$, 
	\begin{equation*}
	\log\!\Big(T/K\big(1+\log^2(T/K)\big)\Big)/n\leq \kl(\mu_a+\delta,\mu^*-\delta)\;.
	\end{equation*}
	We cut the sum in \eqref{eq:bound_B_n} at $n(\delta)$, so that
	\begin{align}
	B &\leq n(\delta)-1+\sum_{n=n(\delta)}^{T} \P\big(\kl(\hmu_{a,n},\mu^*-\delta)\leq \kl(\mu_a+\delta,\mu^*-\delta) \big)\nonumber\\
	&\leq \frac{\log\!\Big(T/K\big(1+\log^2(T/K)\big)\Big)}{\kl(\mu_a+\delta,\mu^*-\delta)}+\sum_{n=n(\delta)}^{T}\P\big(\kl(\hmu_{a,n},\mu^*-\delta)\leq \kl(\mu_a+\delta,\mu^*-\delta) \big)\,.\label{eq:bound_B_cut}
	\end{align}
	Recall that by assumption $\delta< (\mustar-\mu_a)/3$, using the inclusion 
	\begin{equation*}
	\big\{\kl(\hmu_{a,n},\mu^*-\delta)\leq \kl(\mu_a+\delta,\mu^*-\delta)\big\} \subseteq \{\hmu_{a,n} \geq \mu_a+\delta\}\,,
	\end{equation*}
	together with Inequality~\eqref{eq:hoeffding}, we obtain that
	\begin{multline*}
	\sum_{n=n(\delta)}^{T}\P\big(\kl(\hmu_{a,n},\mu^*-\delta)\leq \kl(\mu_a+\delta,\mu^*-\delta) \big)\ \leq \sum_{n=n(\delta)}^{T}\P\big(\hmu_{a,n}\geq \mu_a+\delta\big)\\
	\leq \sum_{n=1}^{\infty} e^{- n \delta^2/(2V)}
	= \frac{1}{e^{\delta^2/(2V)}-1}\leq \frac{2 V}{\delta^2}\,,
	\end{multline*}
	and Equation~\eqref{eq:bound_B_cut} yields
	\begin{equation}
	B\leq \frac{\log(T)}{\kl(\mu_a+\delta,\mu^*-\delta)}+ \frac{\log\!\Big(1/K\big(1+\log^2(T/K)\big)\Big)}{\kl(\mu_a+\delta,\mu^*-\delta)}+\frac{2 V}{\delta^2}\,.\label{eq:bound_B_asymp}
	\end{equation}
	\textbf{Conclusion of the proof.} It just remains to plug Inequalities~\eqref{eq:bound_A_asymp} and \eqref{eq:bound_B_asymp} into Equation~\eqref{eq:decomposition_ENa}: 
	\begin{equation*}
	\E \big[ N_a(T)\big ]\leq \frac{\log(T)}{\kl(\mu_a+\delta,\mu^*-\delta)}+ \frac{\log\!\Big(1/K\big(1+\log^2(T/K)\big)\Big)}{\kl(\mu_a+\delta,\mu^*-\delta)}+(16e^2+2)\frac{2 V K}{\delta^2}+1\,,
	\end{equation*}
	and we obtain Equation~\eqref{eq:asymptotic_bound}. Choosing $\delta$ of order $1/\loglog(T)^{1/2}$ yields the asymptotic optimality.

	\section{Conclusion and Perspectives}\label{sec:conclusion}
	We have proved that the kl-UC$\Bpp$ algorithm is both minimax- and asymptotically-optimal for the exponential distribution families described in Section~\ref{sec:settings}.
	So far, this algorithm requires the horizon $T$ as a parameter: to keep the proofs clear and simple, we have deferred to future work the analysis of an anytime variant. 
	We believe, though, that obtaining such an extension should be possible by using the tools developed in~\citet{Degenne:2016:AOA:3045390.3045558}. 
	In addition, we have focused in this paper on asymptotic optimality without trying to derive explicit finite-time bounds: we believe that this would have impaired the clarity and simplicity of the reasoning. But it is certainly a challenging and important objective to design a general strategy that would, in addition to minimax- and asymptotic optimality, would also reach the important notion of \emph{finite-time instance near optimality} of~\citet{lattimore2015optimally}.
	
	From a more technical point of view, it may be possible to suppress the extra  $\log^2$  exploration term in the definition of the confidence bonus~$g$ in Equation~\eqref{eq:def_g}. This is carried out in~\citet{garivier2016onexplore} using some particularities of the Gaussian distributions; using an improved Chernoff bound such as~\citet{talagrand1995missing} may allow considering more general cases. Finally, we defer the consideration of general bounded probability distributions (with non-parametric upper-confidence bounds) to future work.
	
	\acks{This work was partially supported by the CIMI (Centre International de Math\'ematiques et d'{In\-for\-ma\-tique}) Excellence program.
		The authors acknowledge the support of the French Agence Nationale de la Recherche (ANR), under grants ANR-13-BS01-0005 (project SPADRO) and ANR-13-CORD-0020 (project ALICIA).}

	\appendix
	
	\section{Some Technical Lemmas} \label{app:technical}
	
	\begin{lemma}
		For all $\beta>1$ we have 
		\begin{equation}
		\frac{1}{e^{\log(\beta)/\beta} -1}\leq 2 \max\big(\beta, \beta/(\beta-1)\big)\,.
		\label{eq:bound_max}
		\end{equation}
		\label{lem:bound_max}
	\end{lemma}
	\begin{proof}
		Inequality~\eqref{eq:bound_max} is equivalent to 
		\begin{align*}
		e^{\log(\beta)/\beta}-1\geq \frac{1}{2\beta} \min(1,\beta-1)\,.
		\end{align*}
		If $\beta \geq 2$, then
		\begin{equation*}
		e^{\log(\beta)/\beta}-1\geq e^{\log(2)/\beta}-1\geq \frac{\log(2)}{\beta}\geq \frac{1}{2\beta}\,.
		\end{equation*}
		Otherwise, if $1<\beta<2$, as the function $\beta \mapsto \log(\beta)/(\beta-1)$ is non-increasing one gets
		\begin{equation*}
		\frac{\beta}{\beta -1} \Big(e^{\log(\beta)/\beta}-1\Big)\geq \frac{\log(\beta)}{\beta-1}\geq \log(2)\geq 1/2\,.
		\end{equation*}
	\end{proof}
	
	\begin{lemma}(Maximal Inequality)
		\label{lem:maximal_inequality}
		Let $N$ and $M$ be two real numbers in $\R^+\times\overline{\mathbb{R^{+}}}$, let $\gamma$ be a real number in $\R^{+*}$, and let $\hmu_n$ be the empirical mean of $n$ random variables i.i.d. according to the distribution $\nu_{{b'}^{-1}(\mu)}$. Then
		\begin{equation}
		\label{eq:maximal_inequality_for_d}
		\P\big(\exists N\leq n\leq M,\ \  \kl_+(\hmu_n,\mu)\geq \gamma\big)\leq e^{-N\gamma}\,.
		\end{equation}
	\end{lemma}
	
	\begin{proof}
		If $\gamma >\kl(\bar{\mu}^-,\mu)$ or $\hmu_n\geq \mu$ the Inequality \eqref{eq:maximal_inequality_for_d} is trivial. Else, there exist two real numbers $z <\mu $ and $\lambda<0$ such that 
		\begin{equation*}
		\gamma=\kl(z,\mu)=\lambda z- \phi_{\mu}(\lambda)\,,
		\end{equation*}
		where $\phi_{\mu}$ denotes the the log-moment generating function of $\nu_{{b'}^{-1}(\mu)}$. Since on the event $\big\{\exists N\leq n\leq M,\ \ \kl_+(\hmu_n,\mu)\geq \gamma\big\}$ one has at the same time
		\[			\hmu_n \leq \mu\,,\quad 
		\lambda \hmu_n-\phi_{\mu}(\lambda) \geq \lambda z-\phi_{\mu}(\lambda)=\gamma \quad\hbox{ and }\quad
		\lambda n\hmu_n-n\phi_{\mu}(\lambda) \geq  N \gamma\,,\]
		we can write that 
		\begin{align*}
		\P\big(\exists N\leq n\leq M,\ \  \kl_+(\hmu_n,\mu)\geq \gamma\big)&\leq     \P\big(\exists N\leq n\leq M,\ \  \lambda n\hmu_n-n\phi_{\mu}(\lambda) \geq N \gamma \big)\\
		&\leq \exp(-N\gamma)\,,
		\end{align*}
		by Doob's maximal inequality for the exponential martingale $\exp\!\big(\lambda n\hmu_n-n\phi_{\mu}(\lambda)\big)$.
	\end{proof}
	As a simple consequence of this Lemma~\ref{lem:maximal_inequality} and Inequality~\eqref{eq:pinsker}, it holds that:
	\begin{align}
	\label{eq:hoeffding}
	\text{for every $x\leq \mu$,}\qquad      \P(\exists N\leq n\leq M,\ \ \hmu_n \leq x)&\leq e^{- N (x-\mu)^2/(2V)}\,, \\
	\text{for every $x\geq \mu$,}\qquad         \P(\exists N\leq n\leq M,\ \ \hmu_n \geq x)&\leq e^{- N (x-\mu)^2/(2V)}\,.
	\end{align}
	
	\bibliography{biblio-BLB}
\end{document}